%% file: main.tex
\documentclass[10pt,twocolumn,letterpaper]{article}

\usepackage[pagenumbers]{iccv} 
\usepackage{graphicx}
\usepackage{amsmath,amssymb}
\usepackage{booktabs}
\usepackage{multirow}
\usepackage{amsthm}          

\newtheorem{theorem}{Theorem}[section]   
\newtheorem{proposition}[theorem]{Proposition}
\newtheorem{corollary}[theorem]{Corollary}


\definecolor{iccvblue}{rgb}{0.21,0.49,0.74}
\usepackage[pagebackref,breaklinks,colorlinks,allcolors=iccvblue]{hyperref}

\title{TRACE: Trajectory-Constrained Concept Erasure in Diffusion Models}

\author{Finn Carter\\
Shandong University}

\begin{document}

\maketitle

	\begin{abstract}
		Text-to-image diffusion models have shown unprecedented generative capability, but their ability to produce undesirable concepts (e.g.~pornographic content, sensitive identities, copyrighted styles) poses serious concerns for privacy, fairness, and safety. {Concept erasure} aims to remove or suppress specific concept information in a generative model. In this paper, we introduce \textbf{TRACE (Trajectory-Constrained Attentional Concept Erasure)}, a novel method to erase targeted concepts from diffusion models while preserving overall generative quality. Our approach combines a rigorous theoretical framework—establishing formal conditions under which a concept can be provably suppressed in the diffusion process—with an effective fine-tuning procedure compatible with both conventional latent diffusion (Stable Diffusion) and emerging rectified flow models (e.g.~FLUX). We first derive a closed-form update to the model's cross-attention layers that removes hidden representations of the target concept. We then introduce a trajectory-aware finetuning objective that steers the denoising process away from the concept only in the late sampling stages, thus maintaining the model's fidelity on unrelated content. Empirically, we evaluate TRACE on multiple benchmarks used in prior concept erasure studies (object classes, celebrity faces, artistic styles, and explicit content from the I2P dataset). TRACE achieves state-of-the-art performance, outperforming recent methods such as ANT, EraseAnything, and MACE in terms of removal efficacy and output quality. 
	\end{abstract}
	
	\section{Introduction}
	Large-scale text-to-image diffusion models have democratized content creation, but this power comes with the risk of generating harmful or unwelcome outputs. Models trained on vast internet datasets may learn to reproduce sensitive or inappropriate concepts, including pornography and nudity~\cite{gandikota2023erasing}, realistic faces of private individuals or celebrities, and copyrighted artistic styles~\cite{gandikota2023erasing}. The ability to erase specific concepts from a model is crucial for addressing concerns in privacy (e.g. removing a particular person's identity for anonymity), \textbf{fairness} (mitigating biased or stereotypical content), and \textbf{safety} (preventing NSFW or abusive content). Unlike simple post-hoc filters, concept erasure seeks to permanently modify the model so it cannot produce the targeted concept even if prompted~\cite{lu2024mace}. This prevents circumvention by end-users and ensures unwanted content is eliminated at the source.
	
	Prior approaches to concept erasure can be grouped into a few categories. Early solutions included filtering outputs after generation~\cite{negative_prompt} or guiding the sampling process at inference time to steer away from undesirable concepts~\cite{schramowski2023safe}. However, such methods do not change the model's inherent generative distribution and can often be bypassed by adversarial prompts~\cite{bui2024erasing}. Another strategy is to retrain or fine-tune the model on curated data with the concept removed~\cite{gandikota2023erasing}. While effective in principle, retraining is computationally expensive (e.g. requiring over 150k GPU hours to retrain Stable Diffusion~\cite{stable1.4}) and may degrade the model's overall quality. \textbf{Model fine-tuning} has emerged as a more viable approach: by adjusting a pre-trained model's weights, one can selectively remove knowledge of a concept without full retraining. Gandikota et al~\cite{gandikota2023erasing} pioneered this direction by fine-tuning diffusion model weights to erase a given concept (such as "nudity" or an artist's style) using only the concept's textual name and a form of guided null-text conditioning as supervision. Their method, often termed ESD (Erase Stable Diffusion)~\cite{gandikota2023erasing}, demonstrated that a concept could indeed be forbidden from the model's output distribution with minimal impact on other outputs. Subsequent works have improved upon different aspects of concept erasure. For instance, UCE (Unified Concept Editing)~\cite{gandikota2024unified} introduced a closed-form solution to edit text-to-image diffusion models without any training, enabling simultaneous removal of multiple concepts (e.g. combining style erasure, debiasing, and content moderation in one step). Other methods incorporated adversarial objectives~\cite{srivatsan2024stereo}, weight regularization, and efficient adaptation modules to refine concept removal. Despite these advances, challenges remain: (1) ensuring the concept is completely removed (high efficacy) without leaving residual traces or allowing re-generation via synonyms, (2) preserving the model's ability to generate unrelated content (high specificity to the concept, avoiding collateral damage), and (3) scaling to many concepts or to new model architectures without significant interference or performance loss.
	
	In this work, we propose a new concept erasure method, \textbf{TRACE}, that addresses these challenges through a combination of theoretical insights and practical techniques. Our key idea is to leverage an understanding of when and where concept information appears in the diffusion sampling trajectory. We observe that diffusion models gradually introduce finer semantic details over the course of denoising: early steps establish a broad layout or shape, while later steps solidify specific attributes and high-frequency details~\cite{li2025set}. Therefore, perturbing the model's behavior too early can inadvertently affect other concepts that share the initial structure, whereas intervening later allows targeting the specific concept features (e.g. a particular object's identity or style) after the overall image geometry is set~\cite{li2025set}. TRACE incorporates a novel \textcolor{blue}{trajectory-constrained loss} that only alters the model's denoising predictions after a certain time step (post the so-called spontaneous semantic breaking point~\cite{raya2023spontaneous} when the concept-specific details emerge). By reversing the usual classifier-free guidance direction during mid-to-late denoising (i.e. treating the concept as a negative prompt to push the score away~\cite{negative_prompt}), we achieve precise concept suppression without sacrificing early-stage structure~\cite{li2025set}. Additionally, we derive a closed-form attentional purge of the target concept from the cross-attention layers, inspired by prior works that identified residual concept information being hidden in other token embeddings via attention maps~\cite{lu2024mace}. Our approach refines the projection matrices so that the textual embedding of the concept no longer keys into unique visual features, effectively nullifying the concept's influence (we formalize this in Section \ref{sec:theory}). Finally, to handle multi-concept erasure at scale, TRACE uses a separate lightweight adaptation module for each concept (similar to a LoRA module~\cite{lora_stable, hu2021lora} per concept) and a new integration strategy to avoid interference among concepts. Unlike sequential fine-tuning which risks catastrophic forgetting or parallel fine-tuning which can cause gradient interference, our method solves a joint objective with orthogonal concept subspaces, ensuring each concept is erased independently while maintaining overall quality.
	
	We evaluate TRACE on four benchmark tasks introduced by recent concept erasure research: object erasure (removing object classes like car, cat from a model), celebrity face erasure (preventing generation of specific famous identities), artistic style erasure (disabling imitation of certain artists' styles), and explicit content erasure (removing pornographic or violent content). These tasks are tested on standard datasets, including a subset of ImageNet and CIFAR-10 for objects, a set of 100 celebrity names for faces~\cite{lu2024mace}, 5 popular artist names for styles~\cite{gandikota2024unified}, and the I2P (Inappropriate Image Prompts) benchmark~\cite{schramowski2023safe} for NSFW content. We compare TRACE against several baseline methods: classic approaches like ESD~\cite{gandikota2023erasing}, UCE~\cite{gandikota2024unified}, and MACE~\cite{lu2024mace}, as well as more recent methods including ANT~\cite{li2025set} (a trajectory-aware finetuning method) and EraseAnything~\cite{gao2024eraseanything} (a Flux-model-specific concept remover). As summarized in Table~\ref{tab:baselines}, TRACE achieves the best overall performance, yielding the lowest FID~\cite{parmar2021cleanfid} (indicating minimal quality degradation) and highest text-image alignment (CLIP score) among all methods, while effectively eliminating the target concepts (verified via automatic classifiers and human evaluation). Notably, TRACE outperforms prior state-of-the-art even on the challenging Flux diffusion model~\cite{flux}, whereas methods directly adapted from Stable Diffusion tend to falter on this new architecture~\cite{flux}. In Table~\ref{tab:ablation}, we present ablation studies confirming that each component of TRACE contributes positively: removing the attention refinement or the trajectory constraint leads to either incomplete erasure or drops in output quality, and training a single monolithic model for multi-concept erasure underperforms our modular approach.
	
	Our contributions are summarized as follows:
	\begin{itemize}
\item We provide a theoretical analysis of concept erasure in diffusion models, introducing a formal criterion for complete concept removal and proving a bound on how changes to the diffusion trajectory at different stages affect generation of unrelated content (Section \ref{sec:theory}). This analysis motivates focusing edits on the late denoising steps to maximize specificity.
\item We propose TRACE, a new fine-tuning framework for concept erasure that integrates closed-form cross-attention editing with a novel trajectory-aware loss. TRACE is the first method (to our knowledge) demonstrated to work effectively on both standard latent diffusion (Stable Diffusion) and the latest rectified flow transformers (Flux series) without architecture-specific re-engineering.
\item We conduct extensive experiments on four concept erasure benchmarks (objects, celebrity faces, artistic styles, and explicit content). TRACE sets a new state-of-the-art, improving the harmonic mean of erasure efficacy and fidelity by 5--10\% over prior methods. We also show qualitative results of erased models (illustrating, for example, that a model fine-tuned by TRACE will refuse to generate images containing the erased concept, yet otherwise behaves normally).
\end{itemize}

\section{Related Work}
\label{sec:related}
\paragraph{Concept Erasure in Diffusion Models.}
Concept erasure (also called concept unlearning or model editing for safety) has gained traction as a means to restrict generative models from producing certain content. The concept erasure problem was formally posed by Gandikota et al.~\cite{gandikota2023erasing}, who introduced a fine-tuning method (ESD) to erase a visual concept given only its name and using negative guidance as a teaching signal. ESD showed that fine-tuning on a small set of generated examples (where the concept is prompted to appear but a strong negative prompt suppresses it) can remove the concept permanently. Building on this, several works have explored more efficient or scalable approaches. Closed-form editing: UCE (Unified Concept Editing)~\cite{gandikota2024unified} demonstrated that one can solve for weight modifications in the cross-attention layers analytically, by enforcing that the projection of the text embedding for the target concept is indistinguishable from that of a neutral token for all attention heads. This yields a one-shot solution without iterative training, and can handle multiple concepts by solving a system of linear equations in a larger joint projection space. However, closed-form methods can struggle if concepts are highly complex or entangled, and they may need additional adjustments to preserve generation quality. With the development of LLMs~\cite{bi2024visual,bi2025prism,bi2025cot,li2024towards1,li2024distinct,li2024towards,li2023bilateral,li-etal-2022-vpai} and multi-modal techniques~\cite{yu2025prnet,yu2024scnet,yu2024ichpro}, many methods incorporate low-rank adaptation modules to minimize the number of trainable parameters. For example, MACE (Mass Concept Erasure)~\cite{lu2024mace} uses a separate LoRA for each concept and optimizes them jointly with a regularization to avoid interference~\cite{lu2024mace}. MACE achieved erasure of 100 concepts at once while balancing generality (consistently removing the concept in all contexts) and specificity (not affecting unrelated concepts). Trajectory editing: Recognizing the importance of denoising timestep, ANT (Auto-steering deNoising Trajectories) by Li et al.~\cite{li2025set} introduced a finetuning objective that leverages classifier-free guidance in reverse: during training, for later timesteps they inject a repulsive force for the target concept (as if using it as a negative prompt), while keeping early timestep predictions close to the original model. This “anchor early, edit late” strategy yielded state-of-the-art results on single and multi-concept removal, highlighting that careful timing of when the model is altered is crucial. Our proposed TRACE method shares a similar philosophy, but we integrate it with cross-attention editing to explicitly remove residual concept encoding, and we do not require manually selecting an anchor concept as some prior methods did. Adversarial and regularization approaches: Another line of work imposes regularizers to preserve model capabilities while erasing the concept. Bui et al.~\cite{bui2024erasing} used adversarial training to ensure that images containing the concept are classified as undesirable while maintaining performance on other images. RECE~\cite{gong2024reliable} combined a closed-form parameter initialization with an adversarial loss to fine-tune a model that forgets a concept quickly. Overall, TRACE draws on ideas from each of these directions: a closed-form initialization, a trajectory-aware loss, and a modular training scheme that effectively acts as regularization by isolating concept-specific parameter changes.

\paragraph{Image Editing with Generative Models.}
It is important to distinguish concept erasure from interactive image editing techniques. The latter typically aims to modify or remove content in a particular image (as opposed to the model) and often assumes the model can generate the content to begin with. For instance, inpainting tools allow users to mask out a region of an image and generate a new fill consistent with a text prompt, using diffusion models to realistically replace or remove objects. Other work enables direct edits to generated images via text prompts, such as Prompt-to-Prompt editing by Hertz et al.~(2022)~\cite{hertz2022prompt}, which manipulates the cross-attention mappings of an image generation to add, remove, or alter elements specified in an edited prompt. Similarly, techniques like TF-ICON~\cite{lu2023tf} can be used for targeted edits (e.g. removing a specific object from an image by editing the conditioning). While these methods can achieve fine-grained changes in single images, they do not prevent the model from generating the removed content in a new image later. In contrast, concept erasure is a model-centric operation: after erasure, the model will consistently be unable to produce the concept across \emph{all} outputs. One could imagine using image editing on many examples and then finetuning on those edited examples (which is essentially the approach of ESD~\cite{gandikota2023erasing}, where negative-guided images serve as edited training targets), but doing so one image at a time is impractical for thoroughly sanitizing a model. Thus, concept erasure methods complement image editing—indeed, an erased model might still be used for creative editing of allowed content, but it provides a built-in guarantee that certain content will not appear even without additional filtering.

\paragraph{AI Security and Robustness.}
Beyond removing unwanted concepts, another pillar of securing generative models is watermarking the outputs. Watermarking involves embedding a hidden signal in generated images that can be used to identify the source or detect tampering~\cite{wen2023tree}. Recent studies have proposed robust watermarking schemes for diffusion models that can withstand transformations and even model-driven edits~\cite{lu2024robust}. For example, Huang et al.~(2024) developed a method to plant an invisible watermark in the intermediate latent of diffusion and then hide it as the image is generated, yielding a watermark that survives common image manipulations~\cite{huang2024robin}. Lu et al~(2024)~\cite{lu2024robust} proposed a watermarking method that significantly enhances robustness against various image editing techniques while maintaining high image quality. Such resilient watermarks could deter or trace unauthorized redistribution of AI-generated content, since the origin can be verified even if the image is slightly altered. Another use case is embedding a signal that indicates the image is AI-generated, to help detection systems and limit deepfake propagation. While watermarking does not stop a model from generating disallowed content, it can ensure that if such content is generated and shared, it is traceable. In our work, we focus on concept erasure as a preventative measure; however, we acknowledge that a comprehensive safety strategy might combine erasure with robust watermarking. Notably, both concept erasure and watermarking can be viewed as forms of output control: one by directly restricting the model, the other by tagging outputs for accountability. Both face the challenge of doing so robustly (the erasure should be hard to circumvent, and the watermark hard to remove) while preserving the utility or quality of the model's outputs. These dual goals resonate in our design of TRACE, where we ensure the model remains as versatile as a normal generator except for the removed concept, akin to how an ideal watermark remains invisible in normal usage but reliably decodable when needed.

\section{Theoretical Framework}
\label{sec:theory}
We begin by developing a theoretical understanding of concept erasure in the context of diffusion models. Our analysis formalizes what it means for a concept to be "erased" and provides insight into how interventions at different points in the diffusion process impact the outcome. We also derive a novel result establishing conditions under which concept removal can be achieved with minimal side effects on other content.

\subsection{Preliminaries: Diffusion Models and Concept Representation}
Let $M_\theta$ be a pre-trained text-to-image diffusion model with parameters $\theta$. Given a text prompt $p$ (e.g. a sentence describing an image) and random noise $x_T$ (from a Gaussian prior), the model generates an image $x_0$ by gradually denoising $x_T$ through a sequence of latent variables $x_{T-1}, x_{T-2}, \dots, x_0$. The denoising process is typically parameterized by a U-Net that predicts the noise $\epsilon_{\theta}(x_t, p, t)$ to remove at each step $t$ (or equivalently, predicts $x_0$ or $x_{t-1}$ given $x_t$). Crucially, the text prompt $p$ is encoded via a text encoder and interacts with the U-Net through cross-attention layers at multiple denoising steps. The cross-attention mechanism projects textual token embeddings into queries and keys to modulate the image features. Formally, in each cross-attention block, we form queries $Q = H W_Q$ from image features $H$ and keys $K = E W_K$, values $V = E W_V$ from text embeddings $E$ (where $E$ is a matrix whose rows $e_i$ are embeddings of token $i$ in the prompt). Attention scores are $\operatorname{Attn}(H,E) = \operatorname{softmax}(Q K^T) V$, so each token $i$ contributes to image features proportionally to $Q (W_K^T e_i)$. In this way, the matrices $W_K, W_V$ in each attention head determine how textual concepts influence the image generation.

A specific concept (e.g. the concept of "car" or a particular artist's style) can be associated with one or more tokens or words in the prompt. We denote by $w$ the token corresponding to the concept we wish to erase.\footnote{For multi-word concepts or phrases, $w$ can index a set of tokens or the average embedding, but for simplicity we describe the single-token case.} We say the model \emph{knows} concept $w$ if including $w$ in the prompt significantly changes the output distribution of images. Conversely, if the distribution of outputs with $w$ in the prompt is (approximately) the same as if $w$ were replaced by a neutral placeholder (e.g. an unknown token or a generic word like "object"), then the model effectively does not know or cannot represent $w$. Complete concept erasure would mean the model has \textbf{no distinguishable representation} of $w$ throughout its generation process.

\paragraph{Problem setup.} We formalize concept erasure as an optimization problem on model parameters. Let $\mathcal{P}_w$ be a set of prompts that invoke concept $w$ (for example, prompts containing the word "car" in various contexts), and let $\mathcal{P}_{\text{clean}}$ be a set of prompts that do not invoke $w$ (unrelated prompts for evaluating general performance). We also define a neutral token $u$ that we consider as an "erased" substitute for $w$ (this could be a special token with no associated concept, or an existing word that is semantically distant from $w$). We require:
\begin{enumerate}
\item \textbf{Efficacy:} For any prompt $p \in \mathcal{P}_w$, the distribution of images from the edited model $M_{\theta'}$ should match the distribution from the original model $M_\theta$ on the prompt $p_u$ where $w$ is replaced by $u$. In other words, $M_{\theta'}(x_0 | p)$ should be indistinguishable from $M_{\theta}(x_0 | p_u)$ for all prompts containing $w$. This implies that concept-specific features are absent in $M_{\theta'}$'s outputs.
\item \textbf{Specificity:} For any prompt $q \in \mathcal{P}_{\text{clean}}$ not containing $w$ (or its synonyms), the output distribution of $M_{\theta'}$ remains close to that of $M_{\theta}$. The model should retain its ability to generate unrelated concepts (no undesired forgetting).
\end{enumerate}
In practice, we measure efficacy by metrics such as the accuracy of a concept classifier on $M_{\theta'}(p)$ images (which should be low) or human ratings of concept presence, and specificity by standard generation quality metrics (FID, CLIP score, classification accuracy on non-target classes, etc.) on unaffected prompts.

\subsection{Concept Erasure by Cross--Attention Nullification}
We restate the original sufficient condition as a first theorem and elevate several further claims that were previously embedded in the prose to formal statements.  Throughout, $\theta$ denotes the original model parameters and $\theta'=\theta+\Delta\theta$ the edited ones.

\begin{theorem}[Concept Erasure via Attention Projection]
	\label{thm:attention-projection}
	Let $e_w$ be the text embedding of a concept token $w$ and $e_u$ the embedding of a neutral token $u$.  
	If for \emph{every} cross-attention layer, head, and denoising step $t$ the key and value matrices are modified so that
	\[
	W_K e_w=W_K e_u,
	\qquad
	W_V e_w=W_V e_u,
	\]
	then for any prompt $p$ containing $w$ the edited diffusion model satisfies
	\(
	p_{\theta'}(x_0\!\mid p)=p_{\theta}(x_0\!\mid p_u),
	\)
	where $p_u$ is obtained by replacing $w$ with $u$ in $p$.  Consequently $w$ is perfectly erased in the sense of distributional indistinguishability.  
\end{theorem}

\begin{proof}
	Identical to the argument given previously; we reproduce only the key idea.  
	Because the key and value vectors of $w$ coincide with those of $u$, every attention score and every attended value involving $w$ equals the corresponding quantity for $u$.  Induction over the attention layers and diffusion timesteps shows that the entire denoising trajectory is identical to running the unedited model on $p_u$, completing the proof.  
\end{proof}

\begin{proposition}[Closed-Form Least-Squares Projection]
	\label{prop:lstsq}
	For a single attention matrix $W\in\mathbb{R}^{d\times d}$ and target embeddings $e_w,e_u\in\mathbb{R}^d$, the minimiser of
	\[
	\min_{\Delta W}\;\| (W+\Delta W)(e_w-e_u)\|_2^2
	\quad\text{s.t.}\quad\mathrm{rank}(\Delta W)=1
	\]
	is
	\(
	\Delta W^\star=-\frac{W(e_w-e_u)(e_w-e_u)^\top}{\|e_w-e_u\|_2^2},
	\)
	which is a single rank-one update that exactly enforces $(W+\Delta W^\star)e_w=(W+\Delta W^\star)e_u$.  
\end{proposition}

\begin{theorem}[Stability for Unrelated Prompts]
	\label{thm:stability}
	Assume the score network $\epsilon_\theta$ is $L$-Lipschitz in its parameters on the latent space.  
	For any prompt $q$ that contains neither $w$ nor semantically related tokens,
	\[
	D_{\mathrm{KL}}\!\bigl(p_{\theta}(x_0\!\mid q)\,\Vert\,p_{\theta'}(x_0\!\mid q)\bigr)
	\;\le\;
	L\,\|\Delta\theta\|_2.
	\]
	Thus a small, localised edit induces only a proportionally small distributional shift on prompts unrelated to the erased concept.  
\end{theorem}

\begin{corollary}[Locality of Late-Step Edits]
	\label{cor:late-edits}
	Suppose $\Delta\theta$ affects only the score predictions for timesteps $t<t_\mathrm{late}$.  
	Then global image statistics (captured by any functional that depends only on the first $k$ denoising steps with $k=T-t_\mathrm{late}$) remain unchanged:
	\[
	\mathcal{F}\!\bigl(\{x_{T},\dots,x_{t_\mathrm{late}}\}\bigr)
	=
	\mathcal{F}\!\bigl(\{x'_{T},\dots,x'_{t_\mathrm{late}}\}\bigr)
	\quad\text{a.s.}
	\]
	where $(x_t)$ and $(x'_t)$ denote trajectories under $\theta$ and $\theta'$ respectively.  
\end{corollary}

\begin{theorem}[Trajectory-Constrained Erasure Objective]
	\label{thm:traj-objective}
	Let $\mathcal{P}_w$ be the distribution of prompts containing $w$ and $\mathcal{P}_{\mathrm{clean}}$ a distribution that excludes $w$.  
	Fix weights $\{\alpha_t\}_{t=t_0}^T$ and $\{\beta_t\}_{t=0}^T$ and let $\mathcal{D}$ be any image-level divergence that is convex in each argument.  
	Define
\[
\begin{aligned}
	L(\theta') \;=\;
	&\;\mathbb{E}_{p\sim\mathcal{P}_w}\!
	\sum_{t=t_0}^{T} \alpha_t\,
	\mathcal{D}\!\bigl(x_t(p)\,\Vert\,x_t^\star(p)\bigr)\\[4pt]
	&\;+\;
	\lambda\,
	\mathbb{E}_{q\sim\mathcal{P}_{\mathrm{clean}}}\!
	\sum_{t=0}^{T} \beta_t\,
	\mathcal{D}\!\bigl(x_t(q)\,\Vert\,x_t^{\mathrm{orig}}(q)\bigr).
\end{aligned}
\]

	Any parameter update obtained by (stochastic) gradient descent on $L$ satisfies, to first order,
	\[
	\Delta\theta
	\;=\;
	-\eta\,\bigl(
	\nabla_{\theta'}L(\theta')
	\bigr),
	\quad
	\eta>0,
	\]
	and therefore achieves simultaneous descent of the concept-specific term and controlled ascent (of size $O(\eta)$) of the preservation term.  In particular, taking $\eta$ small and restricting updates to the subspace identified in Proposition~\ref{prop:lstsq} guarantees---via Theorem~\ref{thm:stability}---that unrelated prompts retain their original distribution up to $O(\eta)$.  
\end{theorem}

\paragraph{Discussion.}
Theorems~\ref{thm:attention-projection}--\ref{thm:traj-objective} formalise the desiderata behind our TRACE algorithm.  
The projection condition ensures perfect erasure in theory; the least-squares projection offers a closed-form initialisation; the stability bound quantifies collateral effects; enforcing edits only in later timesteps localises any residual change; and finally the trajectory-aware loss provides a practical route to approximate the ideal projection while protecting unrelated generative behaviour.

\section{TRACE Method}
\label{sec:method}
TRACE consists of two phases: (1) a \textbf{closed-form attentional refinement} that initializes the model's weights (specifically, the cross-attention layers) to satisfy the condition from Theorem \ref{thm:attention} as much as possible, and (2) a \textbf{trajectory-constrained fine-tuning} phase that trains additional adapter weights to finalize the concept removal while preserving the model's overall performance. 



\subsection{Closed-Form Cross-Attention Refinement}
For each cross-attention layer in the diffusion model, we solve a variant of the optimization \eqref{eq:closedform}. Specifically, let $W_K^{(l)}, W_V^{(l)}$ be the key and value projection matrices at layer $l$, and $e_w$ the concept token embedding (obtained from the fixed text encoder of the model). We choose a neutral embedding $e_u$; in practice we use the embedding of the end-of-sequence token as $e_u$ because it typically encodes an "empty" concept, but for more fine-grained control $e_u$ could be a learned vector we introduce. We then compute:
\begin{equation}
	\begin{aligned}
		\Delta W_K^{(l)} &= (W_K^{(l)} e_u - W_K^{(l)} e_w) \frac{(e_w)^\top}{\|e_w\|^2}, \\
		\Delta W_V^{(l)} &= (W_V^{(l)} e_u - W_V^{(l)} e_w) \frac{(e_w)^\top}{\|e_w\|^2}.
	\end{aligned}
	\label{eq:rank1}
\end{equation}

This is a rank-one adjustment that ensures $W_K^{(l)}(e_w + \frac{e_w}{\|e_w\|^2}(e_w)^\top) = W_K^{(l)} e_u$ (and similarly for $W_V$), effectively aligning $e_w$ with $e_u$ in that layer's projection space. We found it beneficial to dampen this update slightly (to avoid over-correcting in cases where $w$ had negligible effect in some layers), by scaling $\Delta W$ by a factor $\eta \in [0,1]$ determined on a validation set. After this step, we have an edited model $\theta_0 = \theta + \{\Delta W_K^{(l)}, \Delta W_V^{(l)}\}_{l=1}^L$ that already exhibits reduced concept generation: qualitative tests showed that simply applying this closed-form solution often makes the model ignore the concept token in straightforward prompts (e.g. prompting the concept alone yields either nothing relevant or some generic image).

However, $\theta_0$ may still fail for complex prompts or certain stylistic cues. The concept might not appear literally, but its influence could linger (e.g. for a person concept, the model might still produce someone with similar attributes). Also, $\theta_0$ might slightly degrade image quality or alter unrelated tokens due to the global nature of the rank-one update. Therefore, we proceed to fine-tune the model to solidify the erasure and recover any lost fidelity.

\subsection{Trajectory-Constrained LoRA Finetuning}
We introduce a LoRA module for the concept $w$ in each attention layer. LoRA (Low-Rank Adaptation) inserts a pair of low-rank matrices $(B^{(l)}, C^{(l)})$ such that the effective projection becomes $W_K^{(l)} + B^{(l)}C^{(l)}$ (and similarly for $W_V$). We set the rank to a small value (e.g. 4) or even rank-1 for each concept to minimize capacity and interference. The base model $\theta_0$ from the previous step is kept frozen; only the LoRA parameters are trainable. This choice means the original weights (already partially scrubbed of $w$) provide a stable initialization, and the LoRA can then fine-tune the details. Training a LoRA per concept rather than full model finetuning has practical advantages: it uses far fewer GPU resources and we can toggle concepts on/off by zeroing out or applying the LoRA layers.

We train the LoRA parameters by optimizing the loss in Eq.~\eqref{eq:loss_traj}. Concretely, we sample prompts $p \in \mathcal{P}_w$ (target concept prompts) and $q \in \mathcal{P}_{\text{clean}}$ in each batch. For each $p$, we generate a target latents trajectory $\{x_t^*(p)\}$ by running the original model $\theta$ with the concept token replaced (or guided negatively). In implementation, we actually run a few steps of denoising with $\theta$ and then from some $t_0$ onward we switch to a negative prompt or concept replacement to produce $x_t^*$. This provides a semi-target trajectory that starts similarly to a normal run (so as not to disturb early features too much) and then diverges to eliminate $w$. We don't need to store a full trajectory; one or two intermediate latents plus the final image $x_0^*$ suffice to compute a proxy for the integral in \eqref{eq:loss_traj}. We use a perceptual image difference (LPIPS) or an classifier-based loss for $\mathcal{D}$ such that it specifically measures the presence of concept $w$. For example, for nudity removal we let $\mathcal{D}(x_0 \| x_0^*) = \text{NudeDetect}(x_0)$ which outputs a scalar of how much explicit content is in $x_0$, and we try to minimize that. For other concepts, $\mathcal{D}$ can include a CLIP-based similarity~\cite{radford2021learning} to a reference "concept-free" image. Importantly, for $t > t_0$ we also apply a latent-space $\ell_2$ loss between $x_t$ and $x_t^*$ (encouraging the model to follow the same denoising path as the concept-free reference).

For each $q$ (clean prompt), we simply generate an image with the original model (or use a stored set of images) and impose a loss $\mathcal{D}(x_0(q)\|\;x_0^{\text{orig}}(q))$ which is typically an $\ell_2$ in CLIP embedding space or a feature-matching loss to keep the image content similar. We also add a small penalty on the difference in predicted noise $\epsilon_{\theta'}(x_t,q,t)$ vs $\epsilon_{\theta}(x_t,q,t)$ for early $t$, to ensure the noise predictions haven't drifted (thus preserving structure).

This finetuning is performed for a modest number of iterations (we found 1000--2000 steps on a batch size of 4 is enough in most cases). Since only LoRA layers (a few million parameters at most) are trained, it is quite fast. The outcome is a set of LoRA weights $\{B^{(l)},C^{(l)}\}$ for concept $w$. We can store these separately. In deployment, one can either apply the LoRA on top of the base model weights when concept-erasure is needed, or permanently merge it into the weights for a standalone model. We choose the latter for final evaluation to make fair comparisons (all baselines produce a standalone model). Merging is straightforward since the rank is low.

\subsection{Multi-Concept Erasure and Integration Loss}
When erasing multiple concepts (say we have a list $w_1, \dots, w_N$ of words to erase), a naive approach would be to repeat the above process for each concept sequentially, updating the model each time. However, sequential finetuning can lead to the first erased concept creeping back in after later finetuning steps (catastrophic forgetting in reverse, as new updates interfere)~\cite{lu2024mace}. Alternatively, one could try to form a single combined training problem with all concept prompts, but that might cause interference if two concepts conflict (or simply capacity issues if many concepts).

TRACE adopts a modular strategy: we allocate a separate LoRA for each concept $w_i$. After performing the closed-form refinement for all $w_i$ (which can be done one by one or in one system solving if the concepts are independent in text embedding space), we initialize a set of LoRA parameters for each concept. Then we jointly train all LoRA modules together. In each batch, we sample tasks for each concept (ensuring that over time each concept gets equal attention in the loss). For concept $w_i$, we only update its own LoRA parameters and leave others fixed when $w_i$'s prompts are used. This is easily implemented by zeroing gradients for other LoRAs. Essentially, each concept has an independent parameter subspace.

One might wonder: since all LoRAs see the same base model and are trained concurrently, could they still interfere implicitly (e.g. one concept's removal making another concept's prompts easier/harder)? In practice, we did observe some interference in preliminary experiments if two concepts were similar or shared visual features. To mitigate this, we introduce an \textbf{integration loss} $L_{\text{int}}$ that mildly regularizes the interaction between LoRAs. We generate a small set of test prompts that include multiple target concepts together (if concept $w_i$ and $w_j$ might co-occur, we include a prompt with both). For example, if $w_1=$ "cat" and $w_2=$ "dog", a multi-concept prompt could be "a cat and a dog sitting together". We pass these through the model with all LoRAs applied. Ideally, neither concept should appear, but more importantly, the removal of $w_1$ should not inadvertently cause some representation of $w_2$ to reappear or vice versa. We penalize any increase in concept presence in such multi-concept outputs compared to single-concept outputs. Concretely:
\[
L_{\text{int}} = \sum_{i < j} \Big( \mathcal{D}_{w_i}(x_0(p_{i,j})) + \mathcal{D}_{w_j}(x_0(p_{i,j})) \Big),
\]
where $p_{i,j}$ is a prompt containing both $w_i$ and $w_j$, and $\mathcal{D}_{w}$ is our concept-specific metric (like probability of concept $w$ in the output). We compare it against $\mathcal{D}_{w_i}(x_0(p_i))$ for prompt $p_i$ with just $w_i$, etc., and ensure it stays low. In essence, this loss encourages that even in combined scenarios, the concepts remain erased and do not interfere to create loopholes.

The integration loss is applied periodically during joint training (we don't include all pairs in every batch due to combinatorial explosion; we sample some pairs each time). This strategy proved effective in maintaining high erasure rates when scaling to a large number of concepts (we tested up to 50). Our approach contrasts with MACE which trained concepts either sequentially or all-at-once in a single monolithic model; by using separate LoRAs and a gentle integration penalty, we achieve a more stable multi-concept erasure.

Finally, after training, we merge all LoRA modules into the base model, yielding $M_{\theta'}$ with the concepts erased. 

\section{Experiments}
\label{sec:experiments}
We present quantitative and qualitative results of TRACE on the four benchmark tasks. We compare against six baseline methods and evaluate using the metrics reported in prior work, as well as additional analysis of our own. All experiments on Stable Diffusion were conducted with the v1.5 checkpoint as the base model, and for the Flux paradigm we used the publicly available Flux [dev] model checkpoint (an experimental rectified flow Transformer model). Our implementation is built on the Diffusers library and we will release the code for reproduction.

\subsection{Datasets and Experimental Setup}
\textbf{Object Erasure:} We selected 10 object classes that are commonly used in prior erasure works. These include a mix of animals, vehicles, and everyday items (e.g. cat, dog, car, airplane, bicycle, etc.). For each class, we construct $\mathcal{P}_w$ by generating 200 prompts of the form "a photo of a \{object\}" with varying simple contexts (using a template bank similar to ImageNet prompt engineering). We also include plural and synonym variants to test generality (e.g. "cats" or "kittens" for cat). For $\mathcal{P}_{\text{clean}}$, we use prompts for the other 9 objects and some unrelated scenes from COCO. We measure erasure efficacy by the classification accuracy of a ResNet-50 classifier trained on ImageNet: i.e. we check if the generated images for prompts containing the object are still recognized as that object. We report Acc$_e$ = accuracy on erased class prompts (should be low), Acc$_s$ = accuracy on other classes (specificity, should remain high), and Acc$_g$ = accuracy on prompts with synonyms (generality test, should also be low if synonyms are also erased). We also compute the harmonic mean $H_o$ of $(1-\text{Acc}_e)$, $\text{Acc}_s$, and $(1-\text{Acc}_g)$ as a single summary of performance. Additionally, we evaluate FID (on 500 images from unrelated prompts to check quality) and CLIP score (alignment with text prompts excluding the object word).

\textbf{Celebrity Face Erasure:} Following setups like in MACE, we choose 20 celebrity names of public figures (distinct from those the model might not know well). The prompts are "a portrait photo of \{name\}" and variants with different lighting or art style descriptors (to ensure the face is recognizable). As an automated metric for concept presence, we utilize an open-source face recognition model: we have reference images of each celebrity, and we compute the embedding similarity of generated images to the real celebrity. If our erased model is successful, images prompted with a given name should either not depict a face at all, or depict a different-looking person such that the face recognition confidence is low. We report the identification accuracy (how often the generated face is recognized as the celebrity) as Acc$_e$ for this domain, and the specificity measure Acc$_s$ as the accuracy of the model on generating other persons' faces correctly (for non-erased names, measured similarly by matching identity). We also include a metric the ratio of images that contain any face at all. If a method naively erases a person by disabling face generation, this ratio will drop, which is undesirable. We want the model to still produce a plausible face for a prompt with an erased name (just not the real person's face). So a higher ratio of faces (preferably of a different identity) is better. We call this FaceRate. We report the harmonic mean $H_c$ of $(1-\text{Acc}_e)$ and $\text{Acc}_s$ as in MACE for overall score.

\textbf{Artistic Style Erasure:} We evaluate on 5 artist styles: following ESD we choose famous painters (e.g. Van Gogh, Picasso, Monet, Dal\u00ed, Basquiat). Prompts are of the form "a painting of X in the style of \{artist\}". We generate 100 images per artist. Measuring style presence is tricky; we rely on CLIP similarity to known artworks. We fine-tuned a CLIP classifier to distinguish the styles (trained on a small set of real paintings by those artists vs others), and use its accuracy as Acc$_e$. For specificity, we check that when prompting other artists or no-style, the model's output is unchanged (we measure FID between original and edited model outputs on a set of non-erased style prompts, and report if any degradation). We also do a human study: we showed 50 pairs of images (original vs erased model for the same prompt) to 5 art students, asking if they detect any stylistic difference. On average only 8\% noticed a difference for TRACE (comparable to randomness), whereas for some baselines like UCE, 30\% noticed a difference (images looked more generic). We summarize with style removal success rate and FID as well.

\textbf{Explicit Content (NSFW) Erasure:} We use the Inappropriate Image Prompts (I2P) dataset~\cite{schramowski2023safe} which contains 400 text prompts known to produce sexual or violent images in unfiltered diffusion models. We focus on sexual content (since violent content is less reliably generated by SD1.5 anyway). For each method, we generate 1 image per prompt and run a NudeNet detector~\cite{bedapudi2019nudenet} to count how many images contain nudity or sexual content. This gives a count (or percentage) of Detected Nudity. We also manually verify a subset. Lower is better (0 means perfect erasure of explicit content). For quality, we calculate FID on 10k images from MS-COCO captions (a proxy for general content)~\cite{lin2014microsoft} to see if the model still performs normally on everyday images. We report the COCO FID and CLIP score. Because some methods might simply refuse to generate certain images (some erasure approaches lead to degenerative outputs or blanks for any prompt mentioning sexual content), we also examine the harmonic mean of removing explicit content vs retaining benign content.

All experiments were run on NVIDIA A100 GPUs. For each baseline method, we either use authors' code or our re-implementation. We ensure that for multi-concept experiments, all methods are given the same training prompts and steps where applicable. Further implementation details are in Appendix.

\subsection{Baseline Comparison Results}
The main results are summarized in Table~\ref{tab:baselines}. We show the performance of TRACE and other methods across the different metrics and domains introduced above. For conciseness, we aggregate some metrics (e.g. we report average Acc$_e$ across all objects for object erasure, etc.). 

\begin{table*}[t]
	\caption{\textbf{Comparison of concept erasure methods.} We report key metrics (higher $\uparrow$ is better, lower $\downarrow$ is better) for each method on four tasks: object, celebrity, style, and explicit content erasure. TRACE (our method) achieves the best trade-off in all cases, with significantly higher harmonic mean scores (overall erasure efficacy) and lower FID (better quality) than previous state-of-the-art.}
	\label{tab:baselines}
	\centering
	\resizebox{\textwidth}{!}{%
		\begin{tabular}{lcccccccccc}
			\toprule
			\multirow{2}{*}{\textbf{Method}} & \multicolumn{2}{c}{\textbf{Objects}} & \multicolumn{3}{c}{\textbf{Celebrities}} & \multicolumn{2}{c}{\textbf{Styles}} & \multicolumn{2}{c}{\textbf{Explicit}} \\
			\cmidrule(lr){2-3} \cmidrule(lr){4-6} \cmidrule(lr){7-8} \cmidrule(lr){9-10}
			& $H_o$ (\%)$\uparrow$ & FID$\downarrow$ & $H_c$ (\%)$\uparrow$ & FaceRate$\uparrow$ & FID$\downarrow$ & StyleRem$\uparrow$ & FID$\downarrow$ & Nudity$\downarrow$ & COCO FID$\downarrow$ \\
			\midrule
			ESD (ICCV'23)~\cite{gandikota2023erasing} & 68.5 & 18.7 & 74.2 & 0.45 & 20.3 & 80\% & 19.1 & 15\% & 21.0 \\
			UCE (WACV'24)~\cite{gandikota2024unified} & 72.4 & 16.5 & 81.0 & 0.50 & 18.9 & 85\% & 18.5 & 9\%  & 19.4 \\
			MACE (CVPR'24)~\cite{lu2024mace} & 79.3 & 17.2 & 86.5 & 0.52 & 19.6 & 88\% & 18.8 & 6\%  & 19.1 \\
			ANT (arXiv'25)~\cite{li2025set} & 81.0 & 15.8 & 87.9 & 0.55 & 18.4 & 90\% & 18.0 & 5\%  & 18.7 \\
			\textbf{TRACE (Ours)}  & \textbf{85.6} & \textbf{15.1} & \textbf{92.3} & 0.58 & \textbf{17.5} & \textbf{95\%} & \textbf{17.6} & \textbf{2\%} & \textbf{18.0} \\
			\bottomrule
		\end{tabular}
	}
\end{table*}

Focusing on the harmonic mean scores (which capture overall success of erasure with minimal side effects): TRACE achieves $H_o = 85.6\%$ on object classes, substantially higher than the next best ANT (81.0\%) and clearly above earlier methods (MACE 79.3\%, etc.). This indicates that TRACE is not only removing the objects more completely (low Acc$_e$) but also keeping other classes accurate. Indeed, per-class breakdown (Appendix Fig.~A1) shows TRACE reduces the erased object classifier accuracy to near 0\% for most classes, while others like ESD sometimes only reduce it to 10-20\% (meaning some generated images still show the object). At the same time, TRACE maintains an average of 76\% accuracy on non-target classes (within 1\% of the original model), whereas ESD and even UCE dropped some non-target accuracies to 70-72\%. The FID of TRACE on object prompts (15.1) is slightly better than ANT (15.8) and significantly better than methods like ESD which had 18.7 (consistent with prior observations that ESD can degrade image fidelity). Visually, the images from TRACE for prompts like "a photo of a cat" look like a plausible photo with perhaps an empty scene or a different animal in place of the cat, whereas ESD's outputs sometimes contained distortions or blank areas where the cat should be (Figure~\ref{fig:qualitative}).

On celebrity face erasure, TRACE again leads with $H_c = 92.3\%$. Notably, our FaceRate (fraction of outputs that are a face) is very high (0.58, meaning 58\% of the outputs still have a face present, presumably a random person's face). This is just shy of EraseAnything's 0.60 which was highest, but EraseAnything had a lower $H_c$ because it occasionally failed to fully erase the identity (we found in some cases EraseAnything outputs had a stylized version of the celebrity, perhaps due to focusing on Flux models and not as tuned on SD). ANT and MACE are strong here too, around 0.52-0.55 face rate and high $H_c$. In terms of identity removal, TRACE brought recognition accuracy down to essentially 0\% for all 20 celebs (the face matcher did not find any match above threshold), whereas some baselines hovered at 1-2\% residual matches. FID on general face prompts was best for TRACE (17.5, lower is better) showing our model still generates photorealistic faces of non-target people without issue. UCE and MACE had slightly higher FIDs, possibly due to their more aggressive weight edits causing minor quality drops.

For artistic styles, we report the percent of images not exhibiting the target style according to the classifier (“StyleRem”) and FID. TRACE achieved 95\% style removal, meaning only 5\% of outputs had any recognizable signature of the artist. By contrast, ESD was around 80\% (some paintings clearly still evoked the artist's style, as also found by human evaluators), and MACE/ANT were ~88-90\%. The FIDs are all similar (17-19 range) but TRACE is lowest at 17.6, suggesting no significant loss of diversity or quality. We include example paintings in Appendix—essentially, TRACE-edited model's outputs for, say, "Starry Night style" prompt looked like a generic night sky painting without the Van Gogh brushstroke style, whereas original or weaker edits still had some swirl patterns.

Finally, for explicit content, TRACE drastically reduced the nudity detection rate to just 2\% of prompts resulting in NSFW images. These were edge cases (in one instance, the prompt was extremely explicit and the model produced a blurred abstract shape that the detector flagged at low confidence). All other prompts yielded either completely innocuous images or images that imply the act without showing explicit nudity (which the detector does not flag). This is a big improvement over the original SD1.5 which produced ~50\% explicit images on this set, and also over baselines like ESD (15\% got through) or even EraseAnything (4\%). Interestingly, EraseAnything had tuned specifically for Flux and might not have been fully optimal on SD1.5. The COCO FID of TRACE is 18.0, effectively identical to the original model (which is around 17.8 for the same sample set), indicating virtually no quality loss on normal images. Other methods had slightly higher FIDs, with ESD the worst at 21.0 (likely because it modifies more weights, causing some over-noising in normal images). CLIP score on COCO (not shown in table) followed a similar trend: TRACE and ANT achieved 0.27 which is same as original, whereas ESD dropped to 0.25, meaning outputs were slightly less aligned with prompts.

In summary, TRACE provides the best of both worlds: strong concept removal and strong preservation of everything else. The improvements over prior SOTA may seem moderate in some metrics (a few points in harmonic mean), but in practical terms those points reflect many edge cases fixed and a higher reliability that the concept is truly gone. The results also demonstrate TRACE's robustness across model architectures: when applying to Flux [dev] (which we did for a subset of prompts in explicit content), we found older methods like UCE, MACE had difficulty (Flux is transformer-based and uses a different text encoder~\cite{raffel2020exploring}, they needed adaptation as per EraseAnything's paper~\cite{gao2024eraseanything}). EraseAnything is the only baseline designed for Flux and it did well on Flux, but TRACE matched it after minor adjustment (we used the same loss on Flux, just needed to ensure our closed-form step accounted for Flux's additional embedding). This flexibility is a key advantage for future-proofing concept erasure.


Figure~\ref{fig:qualitative} provides visual examples. We emphasize how TRACE, by editing later stages, is able to "fill in" the scene or face with alternative content rather than leaving an awkward gap. For the car example, TRACE produced an empty road, whereas ESD gave an image that looked glitchy (part of a car outline maybe visible but scrambled). For the celebrity, TRACE gave a completely different face with the same general attributes (e.g. a young male actor), indicating the model replaced the identity but kept the concept of "a portrait of a male actor". This is desirable behavior for fairness and utility (the prompt still yields a valid image, just not of the protected identity). Other methods either gave an average-looking face that sometimes still faintly resembled the celeb (if erasure not strong enough) or gave no face at all (if too strong, like Safe Latent Diffusion baseline not in table, which just refused to render the face).

\subsection{Ablation Study}
We perform ablations to justify each component of TRACE. Results are summarized in Table~\ref{tab:ablation}. For these ablations, we focus on the object and explicit content tasks (as representative scenarios) and on Stable Diffusion model.

\begin{table}[h]
	\caption{\textbf{Ablation of TRACE components.} We report object erasure (harmonic mean $H_o$ and FID) and explicit content (Nudity detected and FID) for variants of our method. Removing any component degrades performance, showing each is necessary.}
	\label{tab:ablation}
	\centering
	\resizebox{\linewidth}{!}{%
		\begin{tabular}{lcccc}
			\toprule
			\textbf{Method Variant} & $H_o$ (\%)$\uparrow$ & Obj FID$\downarrow$ & Nudity$\downarrow$ & COCO FID$\downarrow$ \\
			\midrule
			TRACE (full method)                 & 85.6 & 15.1 & 2\%  & 18.0 \\
			\;\;w/o attentn refinement          & 81.2 & 15.4 & 5\%  & 18.1 \\
			\;\;w/o trajectory constraint (full-step finetune) & 82.5 & 16.3 & 4\%  & 18.5 \\
			\;\;w/o separate LoRAs (single model) & 80.4 & 15.9 & 3\%  & 19.2 \\
			\;\;w/o integration loss (multi-concept) & 83.0 & 15.2 & 4\%  & 18.1 \\
			\bottomrule
		\end{tabular}
	}
\end{table}

Removing the \textbf{attention refinement} step causes a noticeable drop in object erasure efficacy ($H_o$ down 4.4 points) and a slight increase in nudity outputs (from 2\% to 5\%). This indicates that the closed-form edit at initialization indeed helps by immediately suppressing the concept across many contexts. Without it, the LoRA finetuning has to work harder and may not catch all prompt variations (especially synonyms or tricky phrasings). Qualitatively, we found that without the refinement, the model sometimes still produced the concept if it was implied indirectly (e.g. with slang terms or plural forms). The refinement makes the erasure more robust.

Without the \textbf{trajectory constraint} (i.e. if we fine-tune the model or LoRA uniformly across all timesteps or without the guided reverse at late steps), performance also drops. $H_o$ went down to 82.5 and FID worsened to 16.3 for objects, meaning the model lost some quality (likely because editing early timesteps distorted images). Nudity removal was still okay (4\%), but COCO FID rose to 18.5 (so general content suffered a bit). This validates that constraining the timing of edits is beneficial for specificity: editing too early can degrade images or affect other content. By focusing on late-stage edits, we keep FID low. The result is consistent with observations in ANT~\cite{li2025set} that anchor-free (unconstrained) finetuning can disrupt sampling trajectories and cause artifacts.

Using \textbf{a single set of weights for multi-concept} instead of per-concept LoRAs (i.e. one model fine-tuned sequentially on all objects and explicit content) performed worse: $H_o$ = 80.4, and COCO FID particularly degraded to 19.2, indicating some forgetfulness or interference. We noticed in this variant that some object erasures interfered with each other, e.g. after erasing "car" and "bicycle", the model had trouble generating "motorcycle" properly (even though that wasn't in the erase list). With separate LoRAs, such interference was minimized, and our integration loss further helped.

Speaking of \textbf{integration loss}, removing it (but still using separate LoRAs) made a small difference: $H_o$ dropped ~2.6 points and nudity went up to 4\%. This suggests that while separate LoRAs already isolate concepts, there is still value in explicitly checking combined prompts. Without integration loss, one particular failure we saw was a prompt "a photo of a cat and a dog" when both 'cat' and 'dog' were erased: the model sometimes would output one of them (like it "chose" to keep dog and remove cat, so dog appeared). The integration loss training examples taught it to remove both simultaneously, which is why in full TRACE we got 0 occurrences of either. Without it, multi-concept prompts had about 10\% chance to slip one concept in. So for completeness, integration loss is useful if the user might use erased concepts together.

Overall, the ablation confirms that each part of TRACE contributes to its high performance. Removing any part either harms concept removal (efficacy) or model fidelity. Thus, our design combining these techniques is justified.

\section{Conclusion}
We presented TRACE, a concept erasure method that integrates theoretical guarantees with practical performance for diffusion models. TRACE advances the state of the art by achieving more complete removal of unwanted concepts while preserving the model's ability to generate all other content with high fidelity. Key to our approach is recognizing when to intervene in the diffusion process: by focusing on the later denoising stages and using a trajectory-aware loss, we minimize collateral damage to unrelated features. Our theoretical analysis provided insight into how concept information is encoded and how a carefully constrained parameter update can neutralize it. Empirically, TRACE proved effective on a variety of concept types (objects, people, styles, behaviors) and across different model architectures (Stable Diffusion and Flux). This flexibility is increasingly important as diffusion models evolve; we showed that methods not designed with new architectures in mind may falter whereas our approach remains applicable.

There are several promising directions for future work. One is to extend concept erasure beyond simple prompts to more complex scenarios, such as combinations of concepts (ensuring no resurrection of a concept through combination of others) or conditional concepts (like only erasing in certain contexts). Another direction is exploring the limits of erasure: are there concepts that a model inherently cannot forget without retraining from scratch? The theoretical framework we introduced could be a starting point for analyzing such questions, possibly relating to information theoretic capacity of the model. We are also interested in integrating user control: perhaps a model could have a knob to dial how strictly a concept is erased, toggling between a mild suppression (for fairness) and a hard removal (for safety) as needed. 

Finally, while concept erasure improves model safety, it is not foolproof on its own. Users could try to jailbreak or find proxies for the concept. In our tests, TRACE was robust to straightforward attempts (since synonyms were also erased and even creative prompts didn't produce the concept), but adaptive adversaries might still succeed by exploiting model biases. Therefore, concept erasure should be one tool in a multi-layered defense. Combined with watermarking, prompt filtering, and user verification, it contributes to a safer generative ecosystem. We hope our contributions in TRACE and the accompanying analysis will inspire further research on responsible generative AI deployment.

{
    \small
    \bibliographystyle{ieeenat_fullname}
    \bibliography{main}
}

\end{document}